\documentclass[conference]{IEEEtran}
\usepackage{cite}
\usepackage{amsmath,amssymb,amsfonts}
\usepackage{graphicx}
\usepackage{subcaption}
\usepackage{textcomp}
\usepackage{xcolor}

\usepackage{graphicx}
\usepackage{amsmath}
\usepackage{amssymb}
\usepackage{amsthm}
\usepackage{booktabs}
\usepackage{algorithm}
\usepackage{algorithmic}

\usepackage[pagebackref,breaklinks,colorlinks]{hyperref}

\usepackage[capitalize]{cleveref}
\crefname{section}{Sec.}{Secs.}
\Crefname{section}{Section}{Sections}
\Crefname{table}{Table}{Tables}
\crefname{table}{Tab.}{Tabs.}

\newtheorem{lemma}{Lemma}
\newtheorem{proposition}{Proposition}
\newtheorem{corollary}{Corollary}
\newtheorem{remark}{Remark}

\def\BibTeX{{\rm B\kern-.05em{\sc i\kern-.025em b}\kern-.08em
    T\kern-.1667em\lower.7ex\hbox{E}\kern-.125emX}}
\begin{document}

\title{A Gauss-Newton Approach for Min-Max Optimization in Generative Adversarial Networks}


\author{\IEEEauthorblockN{Neel Mishra}
\IEEEauthorblockA{ 
{IIIT Hyderabad}\\
India \\
neel.mishra@research.iiit.ac.in}
\and
\IEEEauthorblockN{Bamdev Mishra}
\IEEEauthorblockA{
{Microsoft} \\
{India}\\
bamdevm@microsoft.com}
\and
\IEEEauthorblockN{Pratik Jawanpuria}
\IEEEauthorblockA{
{Microsoft}\\
{India}\\
pratik.jawanpuria@microsoft.com}
\and
\IEEEauthorblockN{Pawan Kumar}
\IEEEauthorblockA{
{IIIT Hyderabad}\\
India \\
pawan.kumar@iiit.ac.in}
}

\maketitle

\begin{abstract}
A novel first-order method is proposed for training generative adversarial networks (GANs). It modifies the Gauss-Newton method to approximate the min-max Hessian and uses the Sherman-Morrison inversion formula to calculate the inverse. The method corresponds to a fixed-point method that ensures necessary contraction. To evaluate its effectiveness, numerical experiments are conducted on various datasets commonly used in image generation tasks, such as MNIST, Fashion MNIST, CIFAR10, FFHQ, and LSUN. Our method is capable of generating high-fidelity images with greater diversity across multiple datasets. It also achieves the highest inception score for CIFAR10 among all compared methods, including state-of-the-art second-order methods. Additionally, its execution time is comparable to that of first-order min-max methods.
\end{abstract}


\section{Introduction}
Generative modeling has witnessed significant progress in recent years, with three primary categories of methods emerging: diffusion models, auto-regressive models, and generative adversarial networks (GANs). Lately, diffusion-based generative models have outperformed GANs in terms of image generation quality~\cite{dhariwal2021diffusion}. Notably, state-of-the-art diffusion models now trained on ``open world'' datasets such as LAION-5B~\cite{schuhmann2021laion} and models such as CLIP~\cite{radford2016unsupervised} can generate high-quality synthetic images from text prompts. Prominent recent examples include DALLE-2 by OpenAI~\cite{ramesh2022hierarchical}, the open-source stable diffusion model~\cite{rombach2022highresolution}, and Imagen~\cite{saharia2022photorealistic}. A recent auto-regressive model, Pathways Autoregressive text-to-image (Parti), treats text-to-image generation as a sequence-to-sequence modeling problem and employs a transformer-based tokenizer~\cite{yu2022scaling}.

Despite the growing popularity of diffusion and auto-regressive models, \cite{kang2023gigagan} demonstrated that a GAN-based architecture called GigaGAN can scale to large-scale text-to-image synthesis tasks while generating state-of-the-art image quality and being orders of magnitude faster than diffusion models. Furthermore, it supports latent space editing, including style mixing and latent interpolation. Consequently, GANs continue to be a widely-used generative modeling architecture due to their faster image generation and scalability. GigaGAN builds on the generative adversarial network architectures initially introduced in~\cite{goodfellow2014generative}, followed by variants such as Wasserstein GAN (WGAN)~\cite{arjovsky2017wasserstein} and WGAN-GP~\cite{gulrajani2017improved}. Numerous extensions and modifications of GAN architectures have been proposed, including conditional GAN~\cite{Mirza2014} and StyleGAN~\cite{karras2019style}. Contrary to diffusion models, which are known for stable training but slow generation times, training GAN-based architectures involves solving a delicate min-max problem. In this paper, we propose an effective solver for this minmax problem that leads to improved image generation quality and considerably faster training time.

\subsection{The min-max problem in GAN}
In GANs, we model the following min-max problem 
\begin{align}
    \min_x \max_y f(x,y),  \label{minmax}
\end{align}
where $f(x,y)$ is the utility of the min player. Since this is a zero-sum game, $-f(x,y)$ is the utility of the max player. Both players want to maximize their utility function.
For $f(x,y)$ being the utility of the min player we will sometimes denote $g(x,y)=-f(x,y)$ as the utility of the max player. Many works in the literature have tried modeling the utility function $f$. A solution to the min-max equation \eqref{minmax} is a local Nash equilibrium. Specifically, a point $(\bar{x},\bar{y})$ is defined to be a local Nash equilibrium point if there exists an epsilon ball around the point $U,$ such that for all $(x,y) \in U$, we have
\begin{align}
    \label{eq:nash_point}
    f(\bar{x}, y) \leq f(\bar{x}, \bar{y}) \le f(x, \bar{y}).
\end{align}
A simple interpretation of the equation \eqref{eq:nash_point} is that if $(\bar{x}, \bar{y})$ is a local Nash equilibrium point, then any perturbation in $x$ while keeping the $y$ fixed will only increase the function value. Similarly, any perturbation in $y$ while keeping the $x$ fixed can only decrease the function value. 

In the recent work \cite{jin2019}, the authors point out that the local Nash equilibrium may not be a good measure of local optimality. They reason that the Nash equilibrium relates to simultaneous games, whereas, GAN corresponds to a  sequential minimax game. Hence, they suggest using local minimax point as the local optimal point. In Proposition 6 in \cite{jin2019}, they show an example where local minimax point exists, but there are no local or global Nash equilibrium points. Nevertheless, numerous solvers were developed in past few years with a view of considering GAN training as simultaneous minmax problem \cite{mescheder2017,CGDpaper}.

\subsection{Related work on min-max solvers for GANs}
 The gradient descent ascent (GDA) method  in which both players update their utility greedily (by doing descent and ascent ``sequentially" using existing stochastic gradient methods such as SGD, ADAM\cite{kingma2017adam}, others\cite{duchi2011,loshchilov2019decoupled,neel2023}) may lead to a cyclic trajectory around the Nash equilibrium \cite{mazumdar2018} \cite{CGDpaper}.  To counter this cyclic behavior, the work \cite{daskalakis2018training} uses an optimistic variant of GDA in GANs and shows that this variant does not exhibit cyclic behavior and converges to the equilibrium. The authors in \cite{daskalakis2018limit} study the limit points of two first-order methods GDA and Optimistic Gradient Descent Ascent (OGDA); they show that the stable critical points of OGDA are a superset of the GDA-stable critical points, moreover, their update rules are local diffeomorphisms. Implicit regularization has been explored in machine learning in various flavors \cite{arora2019,azizan2019,gunasekar2017,jin2019,ma2017,neyshabur2017}. The authors of \cite{CGDpaper,ICGD} suggest that their solver, Competitive Gradient Descent (CGD) \cite{CGDpaper} incorporates implicit regularization and provide several empirical justification. The paper \cite{mescheder2017,sachin2023} uses the fixed point theory framework to analyze the convergence of common solvers used in GANs, and it shows that the convergence of these solvers depend on the maximum eigen-value of Jacobian of the gradient vector field and on the update rule considered. Moreover, \cite{mescheder2017} proposes a new solver, namely, Consensus Optimization (ConOPT) where they incorporate a regularizer which promotes consensus among both players. The authors of the solver Symplectic Gradient Adjustment (SGA) \cite{balduzzi2018} build up on ConOPT and suggest that anti-symmetric part of the Hessian is responsible for perceiving the small tendency of the gradient to rotate at each point, and to account for it SGA only incorporates the mixed derivative term. Some other variants include changes in normalization's of weights \cite{miyato2018spectral,han2023} or changes in architectures\cite{karras2019stylebased,kang2023scaling}. Table \ref{table:updates} consists of update equations for the first player for various algorithms. We note that except for GDA, all other methods shown use second-order terms.

\begin{table*}[t]
\caption{\label{table:updates}Various update rules for min-max optimization problems. Shown are the update rules for the first player of recent methods. Here, $\eta$ is the step size, and $\gamma$ is gradient correction hyper-parameter.}
\begin{center}
\begin{tabular}[t]{l l}
\toprule 
 Update rule &   Name \\ 
 \midrule 
    $\Delta x$ =  $- \nabla_x f$ &  GDA\\
    $\Delta x$ =  $- \nabla_x f - \gamma \nabla_{xy}^2 f \nabla_y f$ & SGA \cite{balduzzi2018} \\
    $\Delta x$ =  $- \nabla_x f - \gamma \nabla_{xy}^2 f \nabla_y f - \gamma \nabla_{xx}^2 f \nabla_x f$ & ConOpt \cite{mescheder2017}\\
    $\Delta x$ =  $- \nabla_x f - \eta \nabla_{xy}^2 f \nabla_y f + \eta \nabla_{xx}^2 f \nabla_x f$ & OGDA \cite{daskalakis2017} \\
     $\Delta x$ =  $ (Id + \eta^2 \nabla_{xy}^2 f \nabla_{yx}^2 f)^{-1}$   
    $\left( -\nabla_x f - \eta \nabla_{xy}^2 f \nabla_y f \right)$ & CGD \cite{CGDpaper} \\


    $\Delta x$ = $-\nabla_x f + \dfrac{1}{\lambda} \left(\nabla_x f - \dfrac{ (\nabla_x f)(\nabla_x f)^T(\nabla_x f) + (\nabla_x f) (\nabla_y f)^T(\nabla_y f)}{ \lambda +  (\nabla_x f)^T(\nabla_x f) + (\nabla_y f)^T(\nabla_y f)} \right)$ & Ours \\ 
    

    \bottomrule 
    
\end{tabular}
\end{center}
\end{table*}


\subsection{Contributions of the paper:}

\begin{itemize}
    \item We propose a new strategy for training GANs using an adapted Gauss-Newton method, which estimates the min-max Hessian as a rank-one update. We use the Sherman-Morrison inversion formula to effectively compute the inverse.
    \item We present a convergence analysis for our approach, identifying the fixed point iteration and demonstrating that the necessary condition for convergence is satisfied near the equilibrium point. 
    \item We assess the performance of our method on a variety of image generation datasets, such as MNIST, Fashion MNIST, CIFAR10, FFHQ, and LSUN. Our method achieves the highest inception score on CIFAR10 among all competing methods, including state-of-the-art second-order approaches, while maintaining execution times comparable to first-order methods like Adam.
\end{itemize}

\section{Proposed solver\label{sec:proposedSolver}}

For specific loss functions, such as the log-likelihood function, the Fisher information matrix is recognized as a positive semi-definite approximation to the expected Hessian of this loss function~\cite{ng_martens}. It has been demonstrated that the Fisher matrix is equivalent to the Generalized Gauss-Newton (GGN) matrix~\cite{martens2012training,schraudolph2002fast}. From this perspective, the Gauss-Newton (GN) method can be considered an alternative to the natural gradient method for GAN loss functions. Specifically, the GAN minmax objective is derived from the minimization of KL-divergence. Minimizing KL divergence is well-known to be equivalent to maximizing the log-likelihood of the parameterized distribution $p_\theta(x)$ with respect to the parameter $\theta$, i.e., $\inf_{\theta} D_{KL} (p \, || \, p_\theta) = \sup_{\theta} E_{x \sim p}[ \log p_\theta (x)].$ Consequently, using the Gauss-Newton approach is analogous to employing the natural gradient for maximizing the log-likelihood function. Similar to the Fisher matrix, the GN matrix is utilized as the curvature matrix of choice to obtain a Hessian-Free solver. The Gauss-Newton method~\cite{nocedal1999numerical} has been extensively explored in optimization literature. For a comprehensive analysis and discussion on the natural gradient, we refer the reader to~\cite{ng_martens,ollivier2017information,pascanu2013revisiting}. Given the effectiveness of the natural gradient method with the Fisher matrix and its equivalence to the Gauss-Newton method, we aim to develop efficient solvers for minmax problems in GANs. To the best of our knowledge, Gauss-Newton type methods have not been demonstrated to be effective for training GAN-like architectures. In this section, we adapt the Gauss-Newton method as a preconditioned solver for minmax problems and show that it fits within the fixed point framework~\cite{mescheder2017}. As a result, we obtain a reliable, efficient, and fast first-order solver. We proceed to describe our method in detail.
Consider the general form of the fixed point iterate, where $v(p)$ is the combined vector of gradients of both players at point $p = [x, y]^T$ (See Equation \eqref{eqn:vxy}). 
\begin{equation}
    F(p) = p + h A(p)v(p), \label{eqn:fixed}
\end{equation}
where $h$ is the step size. The matrix $A(p)$ corresponding to our method is 
\begin{align} 
A(p) = B(p)^{-1} - I,  \label{eqn:Ap}
\end{align}
where $B(p)$ denotes the Gauss-Newton preconditioning matrix defined as follows
\begin{align} 
B(p) =  \lambda I + v(p) v(p)^T. \label{eqn:Bp}
\end{align}
With Sherman-Morison formula to compute $B(p)^{-1}.$ We have 
\begin{equation}
A(p) = \frac{1}{\lambda} \left(I - \frac{\dfrac{v(p)v(p)^T}{\lambda}}{1 + \dfrac{v(p)^Tv(p)}{\lambda}}\right) - I.
\label{eq:A_mat}
\end{equation}
A regularization parameter $\lambda>0$ is added in \eqref{eqn:Bp} to keep $B(p)$ invertible. Our analysis in  next section suggests to keep $\lambda < 1$ for convergence guarantee. Hence, a recommended choice is $\lambda \in (0,1).$ To compute the inverse of Gauss-Newton matrix $B(p),$ we use the well-known Sherman–Morrison inversion formula to compute the inverse cheaply. Some of the previous second-order methods that used other approximation of Hessian require solving a large sparse linear system using a Krylov subspace solver \cite{CGDpaper,ICGD}; this makes these methods extremely slow. For larger GAN architectures with large model weights such as GigaGAN \cite{kang2023gigagan}, we cannot afford to use such costly second-order methods. 

\begin{remark}
    In standard Guass-Newton type updates, usually there is a ``$-$'' sign in front of $ B(p)$. In contrast, we have ``$+$'' sign. The implication is visible later where our proposed fixed point operator \eqref{eqn:fixed} is shown to be a contractive operator. 
\end{remark}


\begin{algorithm}[t]
\caption{\label{ALG1} Proposed solver for min-max.}
\begin{algorithmic}[1]
\REQUIRE $\min_{x} \max_{y}  f(x,y)$: zero-sum game objective function with parameters $x,y$
\REQUIRE $h$: Step size
\REQUIRE $\lambda$: Fisher preconditioning parameter
\REQUIRE Initial parameter vectors $p_0 = [x_0, y_0]$
\REQUIRE Initialize time step $t \leftarrow 0$
\REQUIRE Compute initial gradient $v_0$ at $p_0$
\REPEAT
\STATE $t\leftarrow t+1$
\STATE $u_t \leftarrow \frac{1}{\sqrt{\lambda}} * v_{t}$
\STATE $z \leftarrow \dfrac{1}{\lambda} * \left(v_t - \dfrac{ u_t ( u_t^T v_t)}{(1 + u_t^Tu_t)} \right)$
\STATE $\Delta \leftarrow - (v_{t} - z)$
\STATE Update \begin{align*}
\begin{bmatrix}
    x_{t} \\
    y_{t} 
\end{bmatrix} \leftarrow \begin{bmatrix}
    x_{t-1}\\
    y_{t-1}\end{bmatrix}  + h * \Delta  
\end{align*}
\UNTIL{$p_t = [x_t,y_t]^T$ converged}
\end{algorithmic}
\end{algorithm}

\subsection*{Algorithm\label{subsec:Algo}}
In Algorithm \ref{ALG1}, we show the steps corresponding to our method, and in Algorithm \ref{ALG2}, we show our method with momentum. In the algorithms, we implement the fixed point iteration \eqref{eqn:fixed} with the choice of matrix given in \eqref{eq:A_mat} above. We use $v_t$ to denote the gradient $v(p_t)$ computed at point $p_t = [x_t, y_t]^T$ at iteration $t.$ We use the notation $u_t$ in Line 3 of Algorithm \ref{ALG1} to denote the gradient $v _t$ scaled by $\sqrt{\lambda}.$ With these notations, our $A(p_t)$ matrix computed at $p_t$ becomes 
\begin{align*}
A(p_t) = \frac{1}{\lambda} \left(I - \frac{u_tu_t^T}{1 + u_t^Tu_t}\right) - I.
\end{align*}
Consequently, the matrix $A(p_t)v_t$ now becomes 
\begin{align*}
A(p_t)v_t = \frac{1}{\lambda} \left(v_t - \frac{u_tu_t^Tv_t}{1 + u_t^Tu_t}\right) - v_t.
\end{align*}
The computation of $A(p_t)v_t$ is achieved in Steps 4 and 5 of Algorithm \ref{ALG1}. Finally, 
the fixed point iteration: $$p_{t+1} = p_t + A(p_t)v_t$$ 
for $p_t = [x_t, y_t]^T$
is performed in Step 6. Similarly, a momentum based variant with a heuristic approach is developed in Algorithm \ref{ALG2}.

\begin{algorithm}[t]
\caption{\label{ALG2} Adaptive solver}
\begin{algorithmic}[1]
\REQUIRE $\min_{x} \max_{y}  f(x,y)$: zero-sum game objective function with parameters $x,y$
\REQUIRE $h$: Step size
\REQUIRE $\beta_2$: Exponential decay rates for the second moment estimates
\REQUIRE $\lambda$: Fisher preconditioning parameter
\REQUIRE Initial parameter vectors $p_0 = [x_0, y_0]$
\REQUIRE Compute initial gradient $v_0$ at $p_0$
\REQUIRE Compute $\theta_0 = v_0^2$
\REQUIRE Initialize time-step: $t \leftarrow 0$ 
\REPEAT
\STATE $t\leftarrow t+1$
\STATE $\theta_{t}\leftarrow \beta_{2}\cdot \theta_{t-1}+(1-\beta_{2}) \cdot v_{t-1}^{2}$
\STATE $g_{t} \leftarrow v_{t} /(\sqrt{\theta_{t}}+\epsilon)$
\STATE $u_t \leftarrow \sqrt{\dfrac{h}{\lambda}} * g_{t}$
\STATE $z \leftarrow \dfrac{1}{\lambda} * \left(v_t - \dfrac{ u_t  (u_t^T  v_t)}{(1 + u_t^Tu_t)} \right)$
\STATE $\Delta \leftarrow - (g_{t} - z)$
\STATE Update \begin{align*}
\begin{bmatrix}
    x_{t} \\
    y_{t} 
\end{bmatrix} \leftarrow \begin{bmatrix}
    x_{t-1}\\
    y_{t-1}\end{bmatrix}  + h * \Delta  
\end{align*}
\UNTIL{$p_t = [x_t,y_t]^T$ converged}
\end{algorithmic}
\end{algorithm}

\section{Fixed point iteration and convergence\label{sec:convergence}}
We wish to find a Nash equilibrium of a zero sum two player game associated with training GAN. As in \cite{mescheder2017}, we define the Nash equilibrium denoted as a point $\bar{p} = [\bar{x}, \bar{y}]^T$ if the condition mentioned in Equation \eqref{eq:nash_point} holds in some local neighborhood of $(\bar{x}, \bar{y}).$ For a differentiable two-player game, the associated vector field is given by 
\begin{align}
    v(x, y) = \begin{bmatrix}
        \nabla_x f(x, y) \\
        - \nabla_y f(x, y)
    \end{bmatrix},   \label{eqn:vxy}
\end{align} 
The derivative of the vector field is 
\begin{equation*}
    v'(x, y) = \begin{bmatrix}
        \nabla_{xx}^{2} f(x, y) & \nabla_{xy}^{2} f(x,y) \\
        -\nabla_{yx}^{2} f(x,y) & -\nabla_{yy}^{2} f(x,y) 
    \end{bmatrix}. 
\end{equation*}
The general update rule in the framework of fixed-point iteration is described in \cite[Equation~(16)]{mescheder2017}. Our method uses the fixed point iteration shown in \eqref{eqn:fixed}.
Recall that $p=(x,y).$ For the minmax problem \eqref{minmax}, with minmax gradient $v(p)$ defined in \eqref{eqn:vxy}, the fixed point iterative method in \eqref{eqn:fixed} well defined.
\begin{remark}
We remark here that in the following a given matrix (possibly unsymmetric) is called negative (semi) definite if the real part of the eigenvalues is negative (or nonpositive), see \cite{mescheder2017}.
\end{remark}

\begin{lemma}
\label{lem:ifandonlyif}
For zero-sum games, $v'(p)$ is negative semi-definite if and only if $\nabla_{xx}^{2}f(x,y)$ is negative semi-definite and $\nabla_{yy}^{2}f(x,y)$ is positive semi-definite.    
\end{lemma}
\begin{proof}
See \cite[Lemma 1 of Supplementary]{mescheder2017}.  $\blacksquare$
\end{proof}

\begin{corollary} \label{cor:negdef}
For zero-sum games, $v'(\bar{p})$ is negative semi-definite for any local Nash equilibrium $\bar{p}.$ Conversely, if $\bar{p}$ is a stationary point of $v(p)$ and $v'(\bar{p})$ is negative-definite, then $\bar{p}$ is a local Nash equilibrium.  
\end{corollary}
\begin{proof}
See \cite[Corollary 2 in Appendix]{mescheder2017}. $\blacksquare$
\end{proof}

\begin{proposition}
\label{prop:fixedpoint}
Let $F: \Omega \rightarrow \Omega$ be a continuously differentiable function on an open subset $\Omega$ of $\mathbb{R}^{n}$ and let $\bar{p} \in \Omega$ be so that
\begin{enumerate}
    \item $F(\bar{p})=\bar{p}$, and
    \item the absolute values of the eigenvalues of the Jacobian $F^{\prime}(\bar{p})$ are all smaller than 1 .
\end{enumerate}
Then, there is an open neighborhood $U$ of $\bar{p}$ so that for all $p_{0} \in U$, the iterates $F^{(k)}\left(p_{0}\right)$ converge to $\bar{p}$. The rate of convergence is at least linear. More precisely, the error $\left|F^{(k)}\left(p_{0}\right)-\bar{p}\right|$ is in $\mathcal{O}\left(\left|\lambda_{\max }\right|^{k}\right)$ for $k \rightarrow \infty,$ where $\lambda_{\max }$ is the eigenvalue of $F^{\prime}(\bar{p})$ with the largest absolute value.
\end{proposition}
\begin{proof}
See \cite[Proposition~4.4.1]{Bertsekas99}.  $\blacksquare$
\end{proof}


The Jacobian of the fixed point iteration \ref{eqn:fixed} is given as follows
\begin{align}
    F'(p) = I + hA(p) v'(p) + hA'(p)v(p) \label{eqn:jacob}
\end{align}
At stationary point $p = \bar{p},$ $v(\bar{p})=0,$ hence from \eqref{eqn:Ap}, we have $A(\bar{p}) = (1/\lambda - 1)I,$ hence, at equilibrium, the Jacobian of the fixed point iteration \eqref{eqn:jacob} reduces to 
\begin{align}
    F'(\bar{p}) = I + h A(\bar{p}) v'(\bar{p}) 
                = I + \sigma v'(\bar{p}), \label{eqn:h}
\end{align}
where 
\begin{align}
\sigma = h(1/\lambda - 1).
\label{eqn:sigma_definition}
\end{align}

\begin{lemma} \label{}
    If $\bar{p}$ is a stationary point of $v(\bar{p})$ and $A(\bar{p})$ is defined as in equation \ref{eq:A_mat}, then $A(\bar{p})v'(\bar{p})$ is negative definite for some $\lambda < 1$.
\end{lemma} 

\begin{proof}
\begin{align}  
     A(\bar{p})v'(\bar{p}) &= \left(\dfrac{1}{\lambda} \left(\dfrac{I - \dfrac{v(\bar{p})v(\bar{p})^T}{\lambda}}{1 + \dfrac{v(\bar{p})^Tv(\bar{p})}{\lambda}}\right) - I \right) v'(\bar{p}). \nonumber
\end{align}
At the fixed point, we have $v(\bar{p}) = 0,$ hence, from above
\begin{align*}
     A(\bar{p})v'(\bar{p}) &= \left(\frac{1}{\lambda} \left(I \right) - I \right) v'(\bar{p}).
\end{align*}
From Corollary \ref{cor:negdef}, $v'(\bar{p})$ is negative definite at the local Nash equilibrium $\bar{p}$, hence, for $\lambda < 1$  the proof is complete. $\blacksquare$
\end{proof}

At the stationary point, we do have $F(\bar{p}) = \bar{p},$ which satisfies the first item of Proposition \ref{prop:fixedpoint}. In the Corollary below, we show that for $\sigma$ in \eqref{eqn:h} small enough, we satisfy the Item 2 of Proposition \ref{prop:fixedpoint}. To this end, we use the following Lemma.

\begin{lemma} \label{lem:real_part}
    Let $U \in \mathbb{R}^{n \times n}$ only has all eigenvalues with negative real part, and let $h>0$ (in Equation \eqref{eqn:sigma_definition}), then the eigenvalues of the matrix $I + \sigma U$ lie in the unit ball if and only if 
    \begin{align}
        \sigma < \dfrac{1}{|R(\xi)|} \dfrac{2}{1 + \left( \dfrac{I(\xi)}{R(\xi)} \right)^2} \label{eqn:tbound}
    \end{align}
    for all eigenvalues $\xi$ of $U.$ Here, $I(\xi)$ and $R(\xi)$ denote the real and imaginary eigenvalues of the matrix $U.$
\end{lemma}
\begin{proof}
    The proof can be found in \cite[Lemma~4 in Supplementary]{mescheder2017}.  $\blacksquare$
\end{proof}

\begin{corollary}
    Let $\bar{p}$ be a stationary point. If $v'(\bar{p})$ has only eigenvalues with negative real part, then the iterative method with the fixed point iteration \eqref{eqn:fixed}, with \eqref{eqn:Ap} and \eqref{eqn:Bp} is locally convergent to $\bar{p}$ for $\sigma$ defined in \eqref{eqn:h}.
\end{corollary}
\begin{proof}
    As mentioned above, $F(\bar{p}) = \bar{p}$ at stationary point $\bar{p}$ thereby satisfying item 1 of Proposition \ref{prop:fixedpoint}. We use Lemma \ref{lem:real_part} with $\sigma$ same as in \eqref{eqn:h} and $U = v'(\bar{p}).$ From corollary \ref{cor:negdef}, we have that $U$ has all eigenvalues with negative real part. Moreover, we can change $h$, and $\lambda$ such that $\sigma$ satisfies the upper bound \eqref{eqn:tbound}, and we will have the eigenvalues of $F'(\bar{p}) = I + \sigma U = I + \sigma v'(\bar{p})$ lie in a unit ball, hence, we also satisfy item 2 of Proposition \ref{prop:fixedpoint}. Thereby, we have local convergence due to Proposition \ref{prop:fixedpoint}.  $\blacksquare$
\end{proof}


\begin{table}[ht]
\centering
	\caption{Generator architecture for MNIST and FashionMNIST experiments. Here Conv1, Conv2, and Conv3 are convolution layers with batch normalization and ReLU activation, Conv4 is a convolution layer with Tanh activation.}
	\begin{tabular}{l|c|c|c|c}
		\toprule
		Module            & Kernel & Stride & Pad & Shape \\
		\midrule
		Input  & N/A    & N/A    & N/A &  $z \in \mathbb{R}^{100} \sim \mathcal{N}(0, I) $ \\
        Conv1 & $4\times 4$  & 1  & 0    & $100 \rightarrow 1024 $  \\
		Conv2 & $4\times 4$  & 2  & 1    & $1024 \rightarrow 512 $ \\
		Conv3 & $4\times 4$  & 2  & 1    & $512 \rightarrow 256 $  \\
		Conv4 & $4\times 4$  & 2  & 1    & $256 \rightarrow 3 $  \\
		\bottomrule
	\end{tabular}%
	\label{table:mnist_generator}
\end{table}
\begin{table}[ht]
\centering 
\caption{Discriminator architecture for MNIST and FashionMNIST experiments. Here Conv1 is a convolution layer with LeakyReLU activation, Conv2, and Conv3 are convolution layers with batch normalization and ReLU activation, Conv4 is a convolution layer with sigmoid activation.}
	\begin{tabular}{l|c|c|c|c}
		\toprule
		Module            & Kernel & Stride & Pad & Shape \\
		\midrule
		Input  & N/A        & N/A    & N/A &  $z \in \mathbb{R}^{100} \sim \mathcal{N}(0, I) $ \\
		Input  & N/A        & N/A    & N/A &  $x \in \mathbb{R}^{3{\times}32{\times}32}$ \\
        Conv1 & $4\times 4$  & 2  & 1    & $3 \rightarrow 256 $  \\
		Conv2 & $4\times 4$  & 2  & 1    & $256 \rightarrow 512 $ \\
		Conv3 & $4\times 4$  & 2  & 1    & $512 \rightarrow  1024$  \\
		Conv4 & $4\times 4$  & 1  & 0    & $1024 \rightarrow 1 $  \\
        \bottomrule
  \hline
	\end{tabular}%
	\label{table:mnist_discriminator}
\end{table}

\section{Numerical results\label{sec:numexp}}
\subsection{Experimental setup}
\begin{figure*}[t]
    \centering
    \begin{subfigure}[t]{0.48\textwidth}
        \centering \includegraphics[height=1.5 in]{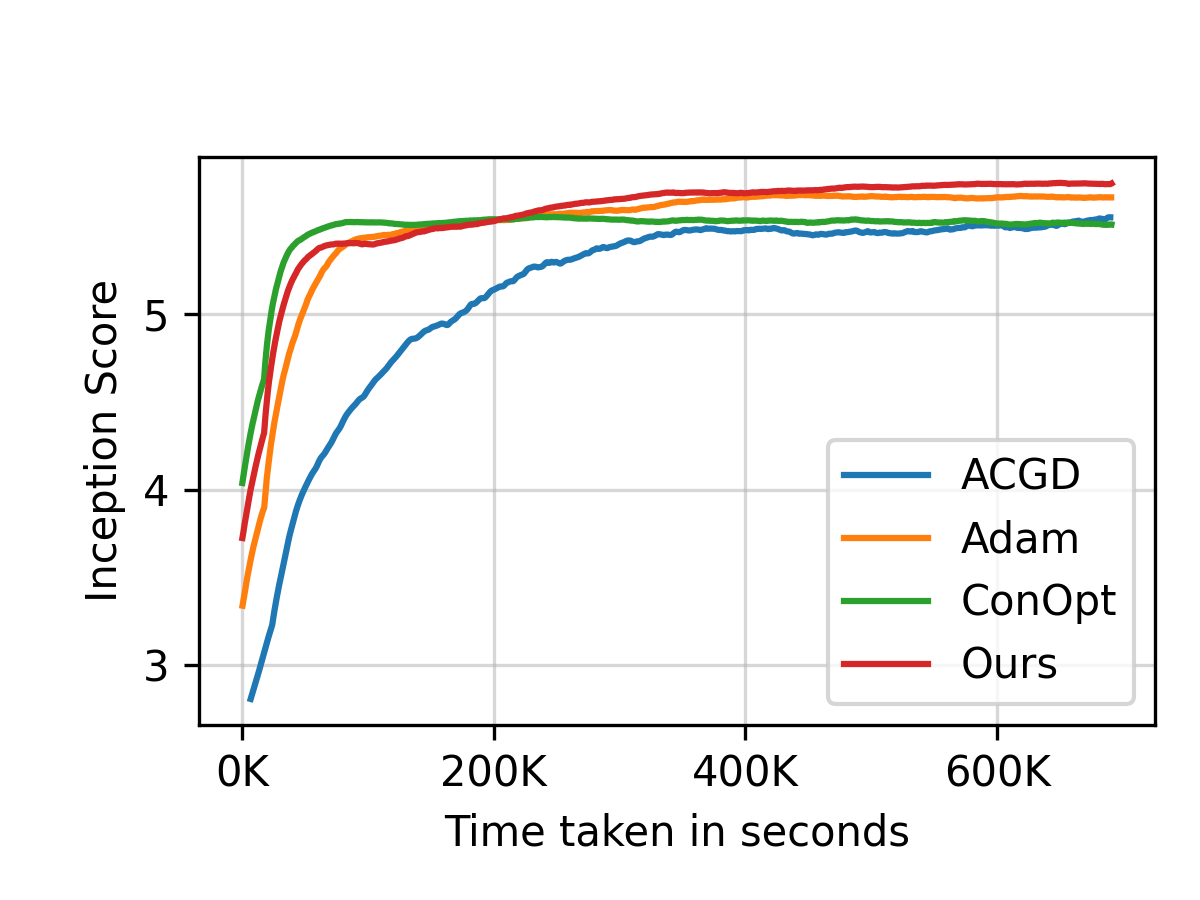}
        \caption{\label{fig:cifarInception}CIFAR10: Inception score over an extended training run.}
    \end{subfigure}
    \begin{subfigure}[t]{0.48\textwidth}
        \centering
        \includegraphics[height=1.5in]{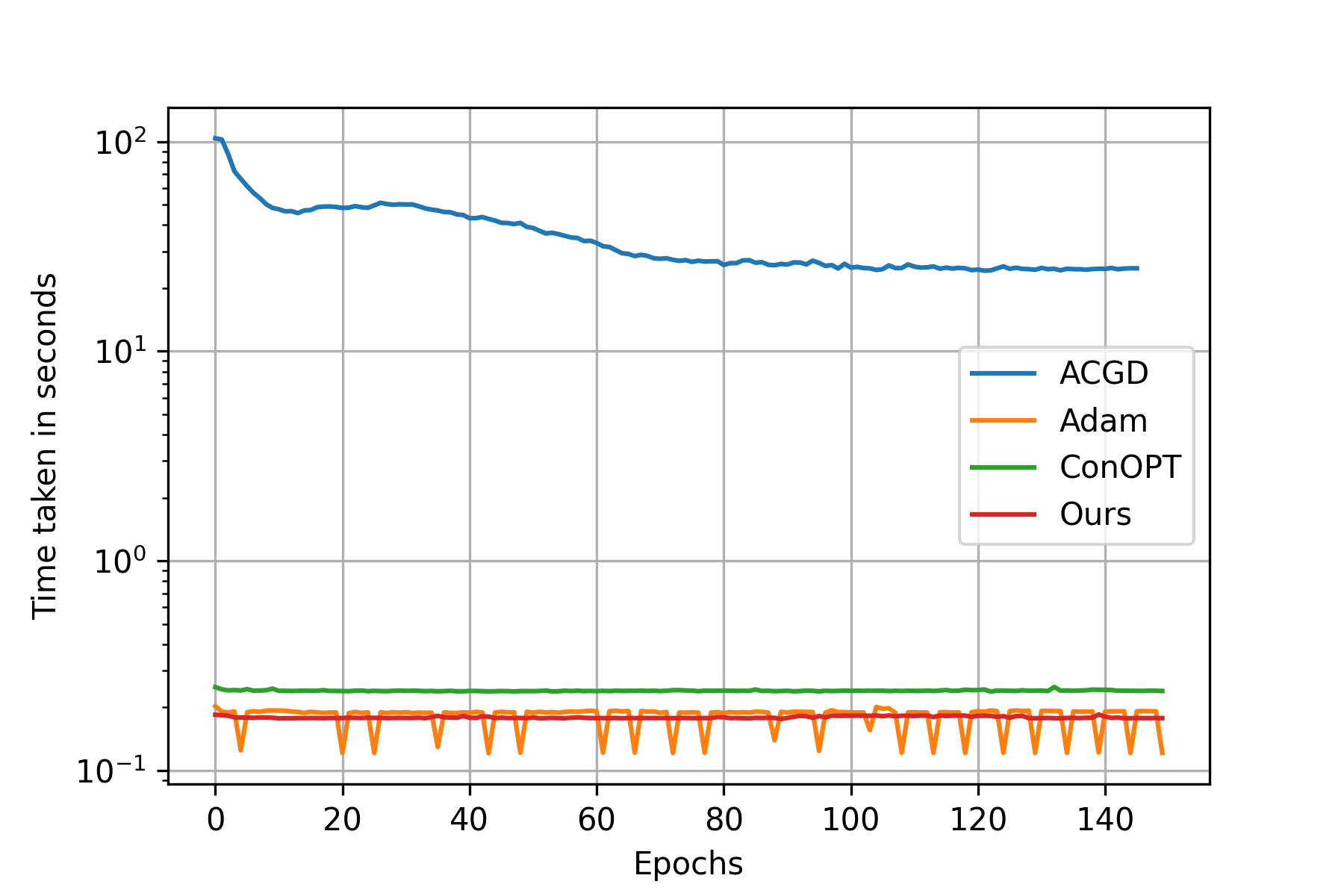}
        \caption{\label{fig:cifarTime}CIFAR10: Time comparison.}
    \end{subfigure}
    \begin{subfigure}[t]{\textwidth}
        \centering
        \includegraphics[height=1.5in]{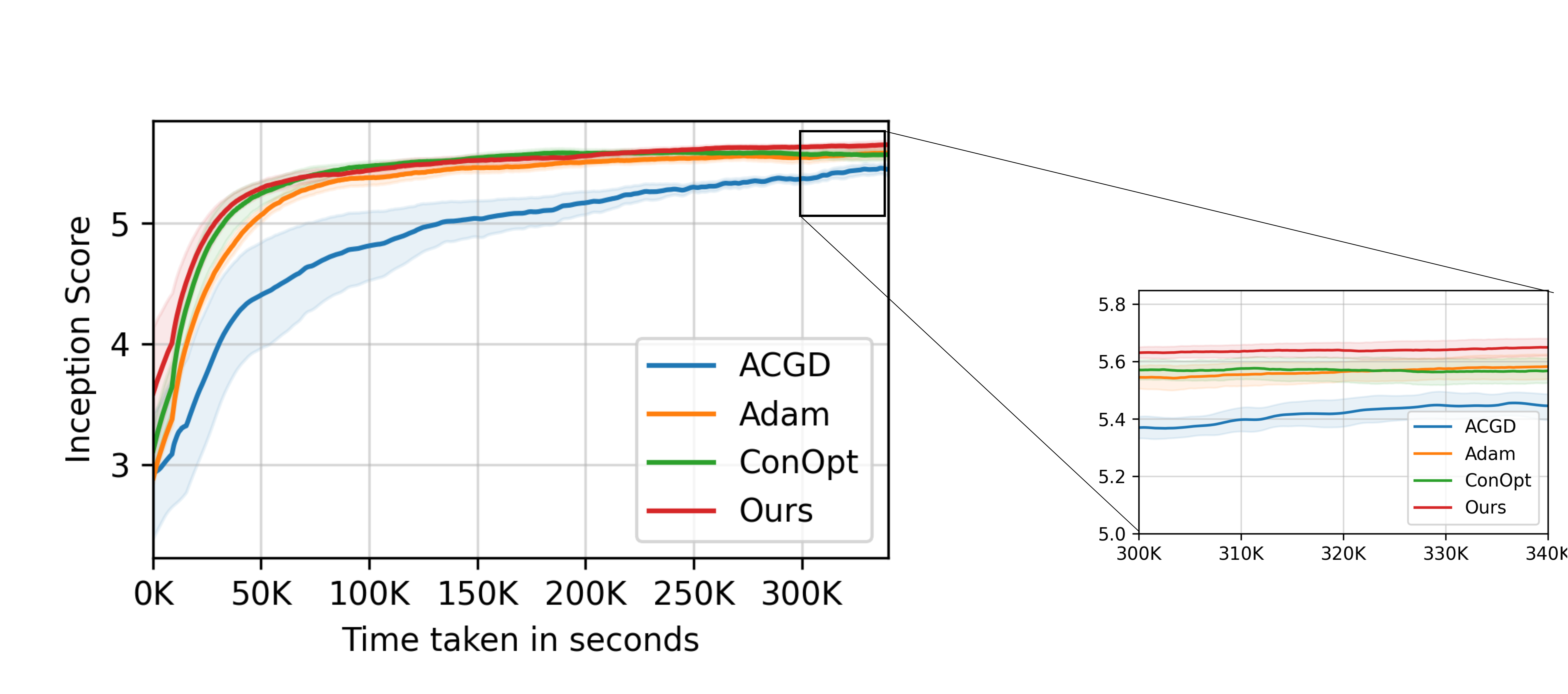}
        \caption{\label{fig:cifarVariance}CIFAR10: Variance in inception scores for multiple runs.}
    \end{subfigure}
    \caption{\label{fig:CIFAR_results}Results for CIFAR10. }
\end{figure*}
In the following, we refer to Algorithm \ref{ALG2} as ``ours''. We compare our proposed method with ACGD \cite{CGDpaper} (a practical implementation of CGD), ConOpt \cite{mescheder2017}, and Adam (GDA with gradient penalty and Adam optimizer). ACGD and ConOpt are second-order optimizers, whereas Adam is a first-order optimizer. We do not compare with other methods including SGA \cite{balduzzi2018} and OGDA \cite{daskalakis2017} as ACGD \cite{CGDpaper} has 
 already been shown to be better than those. 

All experiments were carried out on an \texttt{Intel Xeon E5-2640} v4 processor with 40 virtual cores, 128GB of RAM, and an \texttt{NVIDIA GeForce GTX 1080 Ti} GPU equipped with 14336 CUDA cores and 44 GB of \texttt{GDDR5X VRAM}. The implementations were done using \texttt{Python 3.9.1} and \texttt{PyTorch (torch-1.7.1)}. Model evaluations were performed using the Inception score, computed with the \texttt{inception\_v3} module from \texttt{torchvision-0.8.2}. The WGAN loss with gradient penalty was employed. We tuned the methods on a variety of hyperparameters to ensure optimal results, following the recommendations provided in the official GitHub repositories or papers.
For ACGD, we used the code provided in the official GitHub repository as a solver (\url{https://github.com/devzhk/cgds-package}), and followed their default recommendation for the $maxiter$ parameter. The step size $h$ was tested across multiple values $({10}^{-1}, {10}^{-2}, {10}^{-3}, {10}^{-4}, {10}^{-5}, 8 \times {10}^{-5}, 2 \times {10}^{-4})$, and the gradient penalty constant $\lambda_{\text{gp}}$ was varied as $(10, 1, {10}^{-1}, 0)$. The best result for ACGD in our experimental setup was obtained with $h = {10}^{-4}$ and $\lambda_{\text{gp}} = 1$. In the case of ConOpt, we tried the step size $h$ across ${10}^{-1}, {10}^{-2}, {10}^{-3}, {10}^{-4}, {10}^{-5}, 2\times {10}^{-5}, 8\times {10}^{-6}$, varied $\lambda_{\text{gp}}$ as $(10, 1, {10}^{-1}, 0)$, and adjusted the ConOpt parameter $\gamma$ across $10, 1, {10}^{-1}$. ConOpt performed best in our setup with $h = {10}^{-5}$, $\lambda_{\text{gp}} = 1$, and $\gamma = e^{-1}$. For Adam, we tested the same step sizes as with ConOpt and used the default parameter values recommended by PyTorch. Adam's best performance was achieved with a step size of $lr = {10}^{-5}$. Our proposed method performed best with regularization parameter $\lambda = {10}^{-1}$, step size $h = {10}^{-5}$, and gradient penalty constant for WGAN $\lambda_{\text{gp}} = 10$.

\begin{table}[t]
\centering
	\caption{Generator architecture for CIFAR10 experiments. Here Conv1, Conv2, and Conv3 is a convolution layer with batch normalization and ReLU activation, Conv4 is a convolution layer with Tanh activation.}
	\begin{tabular}{l|c|c|c|c}
		\toprule
		Module            & Kernel & Stride & Pad & Shape \\
		\midrule
		
		Input  & N/A    & N/A    & N/A &  $z \in \mathbb{R}^{100} \sim \mathcal{N}(0, I) $ \\
        Conv1 & $4\times 4$  & 1  & 0    & $100 \rightarrow 1024 $  \\
		Conv2 & $4\times 4$  & 2  & 1    & $1024 \rightarrow 512 $ \\
		Conv3 & $4\times 4$  & 2  & 1    & $512 \rightarrow 256 $  \\
		Conv4 & $4\times 4$  & 2  & 1    & $256 \rightarrow 3 $  \\
		\bottomrule 
	\end{tabular}%
	\label{table:cifar_generator}
\end{table}
\begin{table}[t]
\centering 
	\caption{Discriminator architecture for CIFAR10 experiments. Here Conv1 is a convolution layer with LeakyReLU activation, Conv2 and Conv3 is a convolution layer with batch normalization, and LeakyReLU activation, Conv4 is a convolution layer with sigmoid activation.}
	\begin{tabular}{l|c|c|c|c}
		\toprule
		Module            & Kernel & Stride & Pad & Shape \\
		\midrule
		
		Input  & N/A        & N/A    & N/A &  $z \in \mathbb{R}^{100} \sim \mathcal{N}(0, I) $ \\
		Input  & N/A        & N/A    & N/A &  $x \in \mathbb{R}^{3{\times}32{\times}32}$ \\
        Conv1     & $4\times 4$  & 2  & 1    & $3 \rightarrow 256 $  \\
		Conv2 & $4\times 4$  & 2  & 1    & $256 \rightarrow 512 $ \\
		Conv3 & $4\times 4$  & 2  & 1    & $512 \rightarrow  1024$  \\
		Conv4       & $4\times 4$  & 1  & 0    & $1024 \rightarrow 1 $  \\
		\bottomrule
	\end{tabular}%
	\label{table:cifar_discriminator}
\end{table}


\subsection{Image generation on grey scale images}
\begin{figure*}[H]
    \centering
    \begin{subfigure}[t]{0.48\textwidth}
        \centering \includegraphics[height=1.9 in]{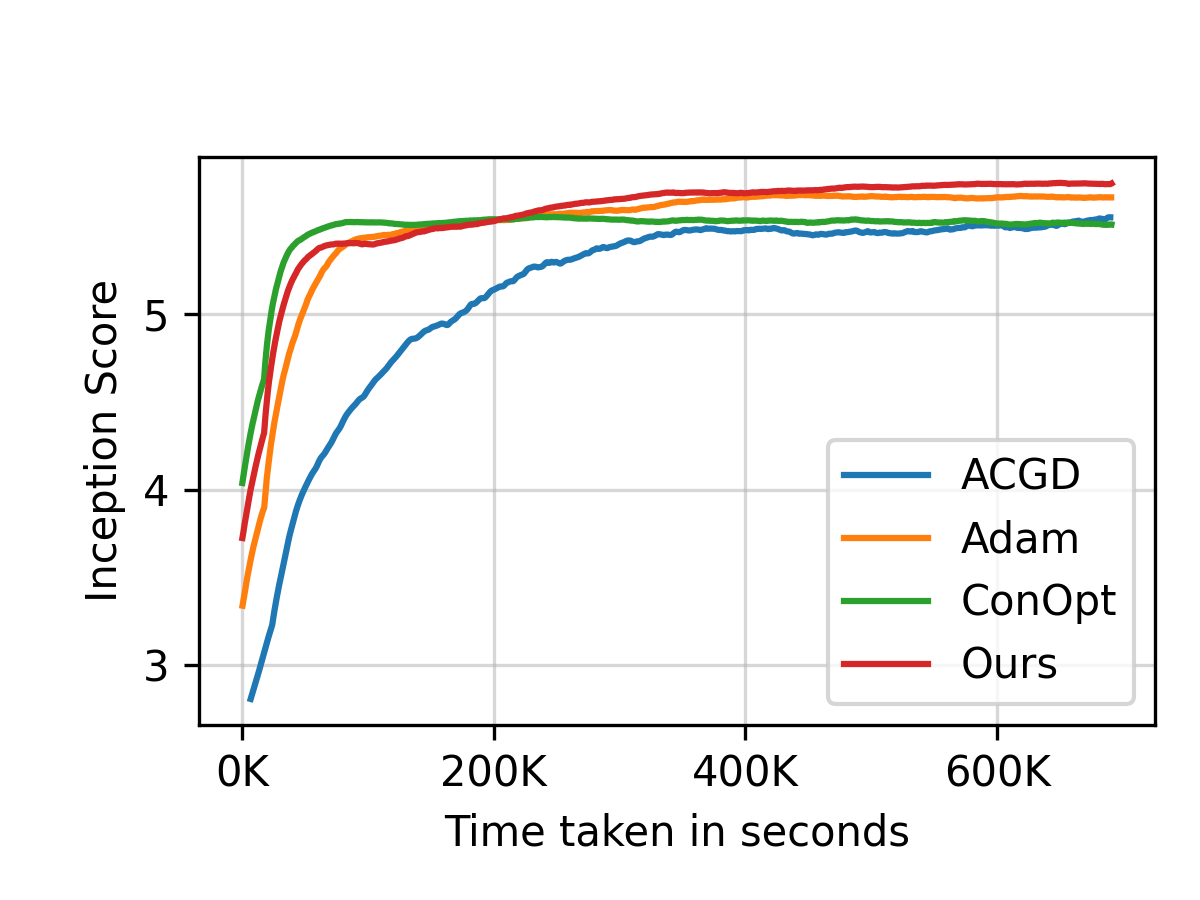}
        \caption{\label{fig:cifarInception}CIFAR10: Inception score.}
    \end{subfigure}
    \begin{subfigure}[t]{0.48\textwidth}
        \centering
        \includegraphics[height=1.8in]{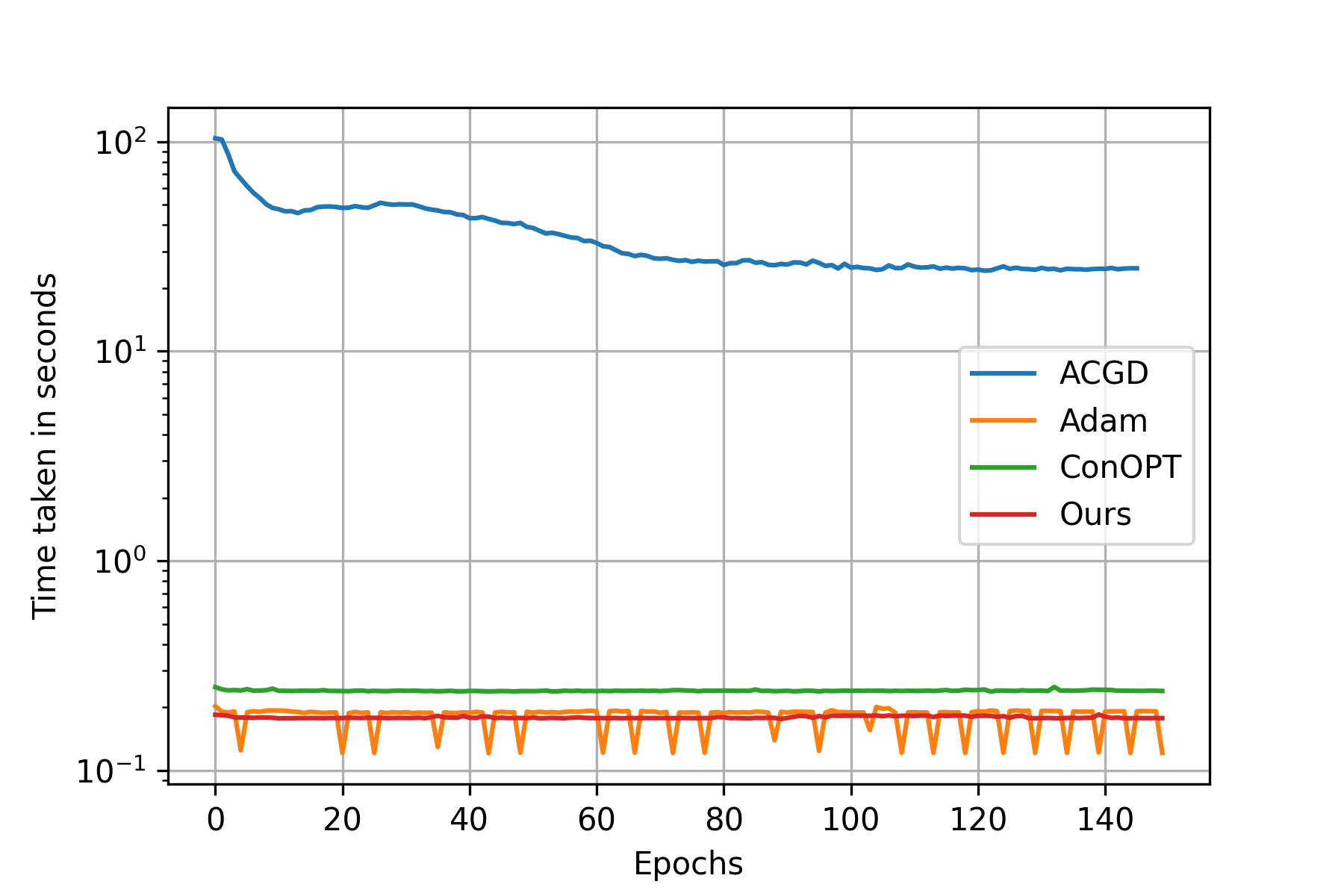}
        \caption{\label{fig:cifarTime}CIFAR10: Time comparison.}
    \end{subfigure}
    \begin{subfigure}[t]{\textwidth}
        \centering
        \includegraphics[height=1.9in]{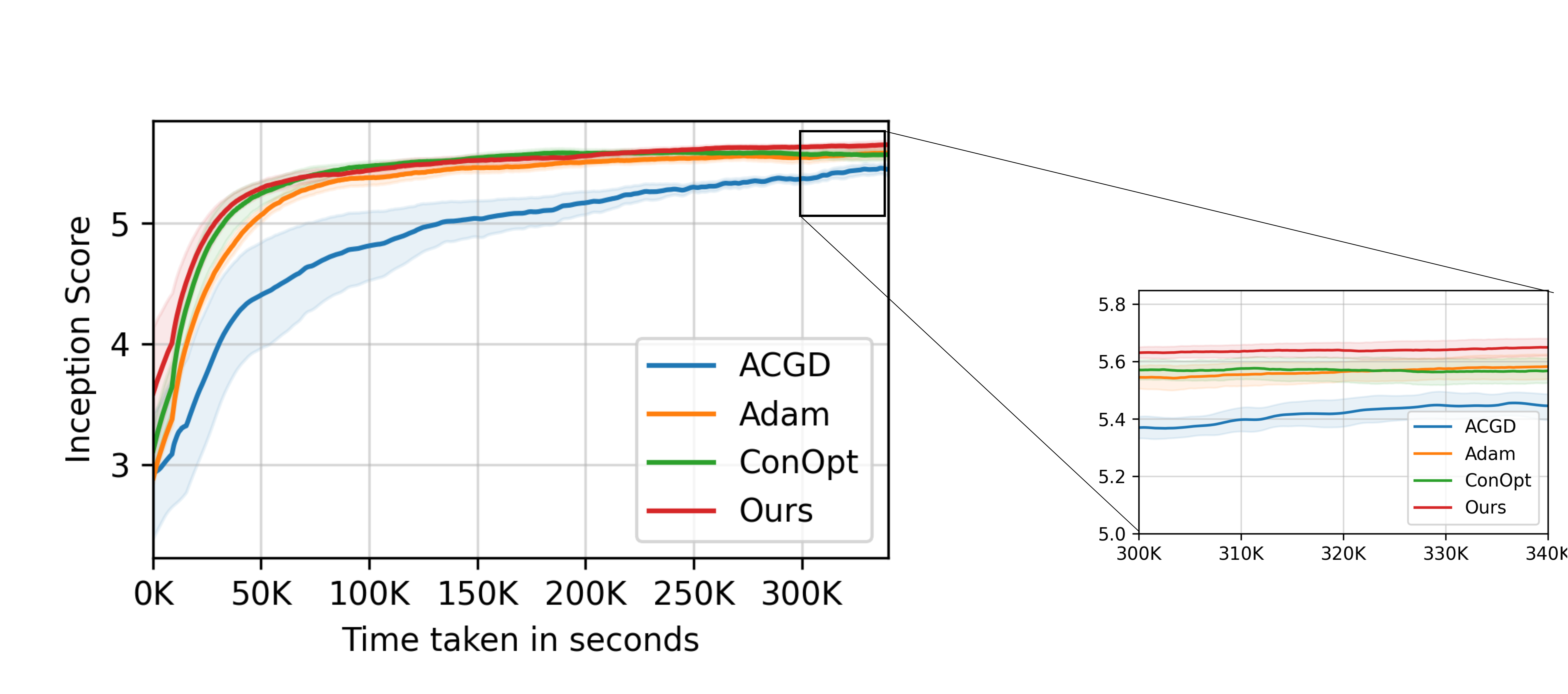}
        \caption{\label{fig:cifarVariance}CIFAR10: Variance in inception scores for multiple runs.}
    \end{subfigure}
    \caption{\label{fig:CIFAR_results}Results for CIFAR10. }
\end{figure*}
We investigate grayscale image generation using public datasets MNIST \cite{726791} and FashionMNIST \cite{xiao2017fashion}, with the discriminator and generator architectures shown in Tables \ref{table:mnist_generator} and \ref{table:mnist_discriminator}, respectively. In Figure \ref{fig:mnist}, we present the MNIST dataset comparison. A careful inspection of the generated images reveals that the first-order optimizer Adam can produce high-fidelity images closely resembling the original dataset shown in Figure \ref{fig:mnist_real}. However, Adam has a drawback: it tends to generate images with characters of the same mode/class, such as 0 and 2, which might indicate mode-collapse. Additionally, Adam maintains a uniform style across all images in Figure \ref{fig:mnist}c. We highlighted the similar samples with white boxes to illustrate mode-collapse for Adam.
The second-order optimizer ACGD offers an alternative solution to this issue, as shown in Figure \ref{fig:mnist_ACGD}. Although the images contain some minor artifacts, they display high quality. ACGD mitigates the consistent character style observed in first-order optimizers. This enables increased diversity in generated images while preserving visual fidelity to the original dataset.
Our proposed method balances the advantages of both Adam and ACGD. By utilizing this approach, we can generate characters in \ref{fig:mnist_b} that closely mimic the original dataset and remain computationally efficient, similar to Adam, a first-order optimizer. Although some minor artifacts may be present in the generated images, our method overall provides a promising solution achieving high-quality and diverse results.
As depicted in Figure \ref{fig:fmnist}, the generation of FashionMNIST images displays a trend similar to that observed in MNIST. For grey-scale images, ConOpt's performance is found to be suboptimal. We note that the ConOpt paper \cite{mescheder2017} did not present results for grey-scale images; thus, we acknowledge that we may not know the appropriate hyperparameters for these datasets.


\begin{figure*}[t!]
    \centering
      \begin{subfigure}[t]{0.18\textwidth}
        \centering
        \includegraphics[height=0.9in]{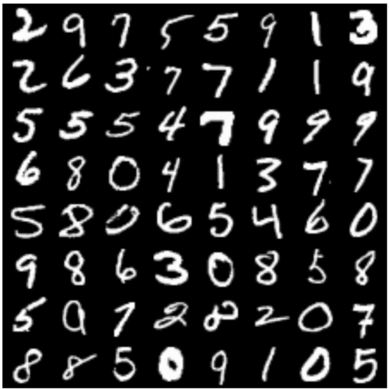}
        \caption{\label{fig:mnist_real}Real}
    \end{subfigure}
    ~
    \begin{subfigure}[t]{0.18\textwidth}
        \centering
        \includegraphics[height=0.9in]{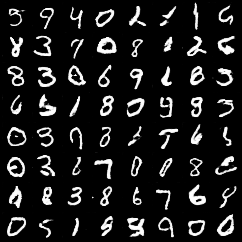}
        \caption{\label{fig:mnist_b}Ours}
    \end{subfigure}%
    ~ 
    \begin{subfigure}[t]{0.18\textwidth}
        \centering
        \includegraphics[height=0.9in]{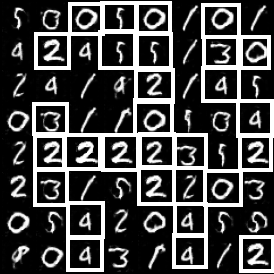}
        \caption{\label{fig:mnist_ADAM}Adam}
    \end{subfigure}
    ~ 
    \begin{subfigure}[t]{0.18\textwidth}
        \centering
        \includegraphics[height=0.9in]{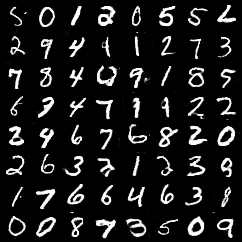}
        \caption{\label{fig:mnist_ACGD}ACGD}
    \end{subfigure}
    ~ 
    \begin{subfigure}[t]{0.18\textwidth}
        \centering
        \includegraphics[height=0.9in]{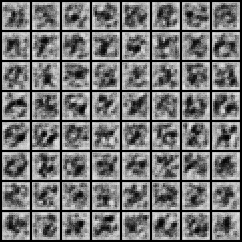}
        \caption{\label{fig:mnist_CONOPT}ConOpt}
    \end{subfigure}
    \caption{\label{fig:mnist}Images generated for MNIST. Samples inside white box show mode-collapse.}
\end{figure*}

\begin{figure*}[t!]
    \centering
      \begin{subfigure}[t]{0.18\textwidth}
        \centering
        \includegraphics[height=0.9in]{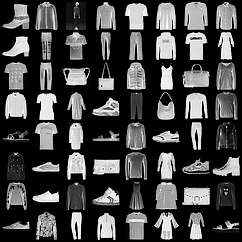}
        \caption{\label{fig:fmnist_REAL}Real}
    \end{subfigure}
    ~
    \begin{subfigure}[t]{0.18\textwidth}
        \centering
        \includegraphics[height=0.9in]{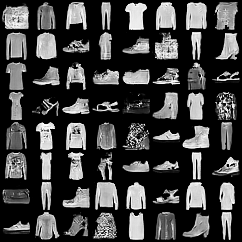}
        \caption{\label{fig:fmnist_Ours}Ours}
    \end{subfigure}%
    ~ 
    \begin{subfigure}[t]{0.18\textwidth}
        \centering
        \includegraphics[height=0.9in]{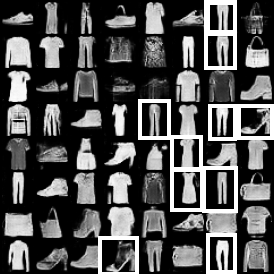}
        \caption{\label{fig:fmnist_ADAM}Adam}
    \end{subfigure}
    ~ 
    \begin{subfigure}[t]{0.18\textwidth}
        \centering
        \includegraphics[height=0.9in]{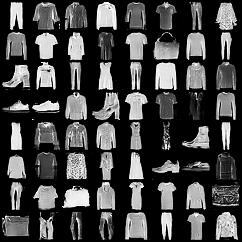}
        \caption{\label{fig:fmnist_ACGD}ACGD}
    \end{subfigure}
    ~ 
    \begin{subfigure}[t]{0.18\textwidth}
        \centering
        \includegraphics[height=0.9in]{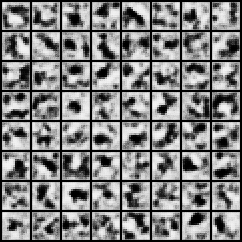}
        \caption{\label{fig:fmnist_CONOPT}ConOpt}
    \end{subfigure}
    \caption{\label{fig:fmnist}Images generated for Fashion MNIST. Samples inside white box show mode-collapse. }
\end{figure*}

\subsection{Image generation on color images}
For color image datasets, we examine the well-known and widely-used datasets {\tt CIFAR10} \cite{cifar10}, {\tt FFHQ} \cite{karras2019style}, and {\tt LSUN} bridges, tower, and classroom datasets \cite{yu2015lsun}. A comparison for the {\tt CIFAR10} dataset \cite{cifar10} is displayed in Figure \ref{fig:CIFAR_results}. Our method generates images of comparable quality to those produced by ACGD. A widely accepted metric for comparing methods on the {\tt CIFAR10} dataset is the inception score\cite{NIPS2016_8a3363ab}, which evaluates the quality and diversity of generated images using a pre-trained Inception v3 model. We analyze the inception scores of ACGD, Adam, ConOpt, and our method in Figure \ref{fig:cifarInception} and Figure \ref{fig:cifarVariance}. The former demonstrates the performance over extended training durations, while the latter illustrates the performance variations across multiple runs. The generator and discriminator architectures for \texttt{CIFAR10} are presented in Tables \ref{table:cifar_generator} and \ref{table:cifar_discriminator}, respectively. Table \ref{tab:time_per_epoch} displays the average time taken per epoch, while Table \ref{tab:inception} showcases the maximum inception scores across multiple runs. Our method achieves the highest inception score of 5.82, followed closely by Adam with a score of 5.76, and then ConOpt and ACGD.

ACGD has a long runtime due to the costs involved in solving linear systems. However, it may achieve a higher inception score if allowed to run for an extended period. To investigate this, we conducted experiments for a long runtime of approximately eight days, as illustrated in Figure \ref{fig:cifarInception}.
We used the same experimental setting as ACGD, including testing different values of GP and a case where GP was not used. For our architecture, ACGD gave better results than the results reported on a smaller architecture in ACGD paper; as shown in \cite[Figure~9]{ICGD}. However, we find that Adam demonstrated better performance than ACGD on this architecture.


Figure \ref{fig:FFHQ_generated_images} displays our generated images of FFHQ, created using a training dataset of FFHQ thumbnails downscaled to 64x64x3 dimensions. Our images demonstrate a high level of detail and texture in the skin and realistic facial features. Figure \ref{fig:Classroom_generated_images} shows our results on the LSUN classroom dataset, where our model successfully learned significant features, capturing the essence of classroom scenes. In Figure \ref{fig:tower_and_bridges}, the generated images of LSUN-tower and LSUN-bridges exhibit recognizable characteristics, such as the distinct shape of a tower and the features typical of bridges.


\begin{figure}[t]
\centering
\begin{subfigure}[t]{0.5\textwidth}
\centering
    \includegraphics[height=1.5in]{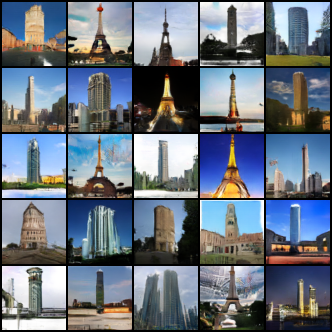}
\caption{\label{fig:tower_generated_images}LSUN-tower.}
\end{subfigure}%

\begin{subfigure}[t]{0.5\textwidth}
\centering
    \includegraphics[height=1.5in]{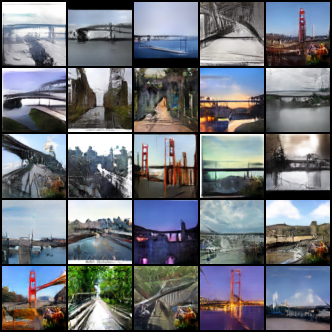}
\caption{\label{fig:bridges_generated_images}LSUN-bridges.}
\end{subfigure}%
\caption{\label{fig:tower_and_bridges}Images generated by our method on LSUN-tower and LSUN-bridges.}
\end{figure}
\begin{figure}[h]
\centering

\begin{subfigure}[t]{0.5\textwidth}
\centering
    \includegraphics[height=1.5in]{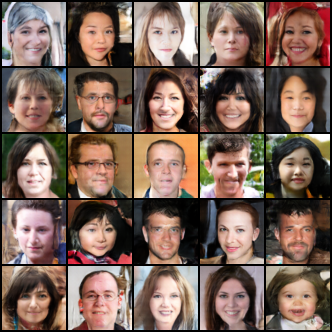}
\caption{\label{fig:FFHQ_generated_images}FFHQ.}
\end{subfigure}%

\begin{subfigure}[t]{0.5\textwidth}
\centering
    \includegraphics[height=1.5in]{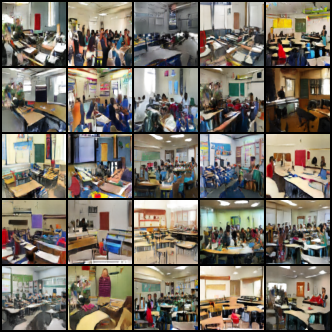}
\caption{\label{fig:Classroom_generated_images} Classroom.}
\end{subfigure}%
\caption{\label{fig:FFHQ_and_Classroom}Images generated by our method on FFHQ and LSUN-classroom.}
\end{figure}

\begin{table}[t]
\begin{center}
\caption{\label{table:Timing_and_inception_scores} Timings  and inception scores across multiple runs.}
    \begin{subtable}[b]{0.48\textwidth}
    \centering
    \begin{tabular}{|l|c|}
      \hline
      Method & Average time per epoch \\
      \hline
      ACGD & 35.34 \\
      Adam & 0.186 \\
      ConOpt & 0.25\\
      Ours & 0.184 \\
      \hline
    \end{tabular}
    \caption{Average time in seconds. }
    \label{tab:time_per_epoch}
  \end{subtable}
  \begin{subtable}[b]{0.48\textwidth}
  \centering
    \begin{tabular}{|l|l|l|l|}
      \hline
      Method & Run1 & Run2 & Run3 \\
      \hline
      ACGD & 5.55 & 5.52 & 5.60 \\
      Adam & 5.38 & 5.65 & 5.58 \\
      ConOpt & 5.73 & 5.72 & 5.76 \\
      Ours & \textbf{5.82} & \textbf{5.77}& \textbf{5.79}  \\
      \hline
    \end{tabular}
    \caption{Max inception scores across runs.}
    \label{tab:inception}
  \end{subtable}
  \end{center}
\end{table}



\subsection{Computational complexity and timing comparisons }
Assuming $x \in \mathbb{R}^m, y \in \mathbb{R}^n,$ then due to the terms $\nabla_{xy}^2$, and $\nabla_{yx}^2$ involved in ConOpt, there is additional computational cost of the order $O(m^2n^2)$, $O(m^3)$ and $O(n^3)$, these costs are associated with constructing these terms and for matrix-vector operation required. On the other hand, for ACGD, there are additional cost of order $O(m^3n^3)$ for solving the linear system with matrix of order $mn \times mn$ if direct method is used, and of order $O(mn)$ if iterative methods such as CG as in \cite{CGDpaper} is used. For empirically verifying the time complexity, for {\tt CIFAR10} dataset, in Figure \ref{fig:cifarTime}, we observe that ACGD is the slowest, and ConOpt is also costlier than our method. Our method has similar complexity to that of Adam the first order methods. To compute the order of complexity of Algorithm \ref{ALG1}, we have $m+n$ cost for scaling, $2(m+n)$ cost for $u(u^Tg_t)$ computation, and $m+n$ cost for $u^T u$ computation, then subtraction with $g_t$ costs $(m+n)$. The total cost per epoch is $O(m+n),$ which is same order as any first-order method such as Adam.



Table \ref{fig:cifarTime} displays a comparison of the time taken per epoch for various methods, including ACGD. It is noteworthy that ACGD consumes a significantly greater amount of time compared to other methods. Specifically, ACGD takes 35 seconds per epoch, which is approximately 200 times more than our method that takes only 0.184 seconds per epoch. Notably, this is almost identical to the time taken by Adam, which is 0.186 seconds per epoch.

\section{Code Repository}
The code for the experiments is available in an anonymous GitHub repository, accessible at:  \url{https://github.com/NeelMishra/Gauss-Newton-Based-Minimax-Solver}.


\section{Conclusion\label{sec:conclusion}}

We have introduced a novel first-order method for solving the min-max problems associated with generative adversarial networks (GANs) using a Gauss-Newton preconditioning technique. Our proposed method achieves the highest inception score among all state-of-the-art methods compared for CIFAR10, while maintaining a competitive time per epoch. We evaluate our method on popular grayscale and color benchmark datasets and demonstrate its superior performance. Furthermore, we prove that the fixed-point iteration corresponding to our proposed method is convergent. As a future work, one may look at more powerful architectures for generators and discriminators that will likely increase the quality of images generated.   

\section{Acknowledgment} This work was completed and funded at IIIT-Hyderabad, India under a Microsoft Academic Partnership (MAPG) Grant.

\bibliographystyle{plain}
\bibliography{mybib}
\end{document}